\documentclass[11pt]{article}

\usepackage[a4paper,margin=1in]{geometry}
\usepackage[T1]{fontenc}
\usepackage[utf8]{inputenc}
\usepackage{lmodern}
\usepackage{amsmath,amssymb,amsthm,mathtools}
\usepackage{bm}
\usepackage{mathrsfs}
\usepackage{hyperref}
\usepackage{enumitem}
\usepackage{booktabs}
\usepackage{microtype}
\usepackage{graphicx}
\usepackage{subcaption}
\usepackage{float}
\usepackage{multirow}
\usepackage{array}
\usepackage{xcolor}

\hypersetup{
    colorlinks=true,
    linkcolor=blue,
    citecolor=blue,
    urlcolor=blue
}

\newtheorem{theorem}{Theorem}
\newtheorem{lemma}{Lemma}
\newtheorem{proposition}{Proposition}
\newtheorem{corollary}{Corollary}
\newtheorem{assumption}{Assumption}
\theoremstyle{remark}

\title{Robustness of Similarity-based Positional Encoding Under Rotations:\\ Theoretical Analysis and Experimental Validation}
\author{
  Andrea Santomauro \and
  Luigi Portinale \and
  Giorgio Leonardi \\[0.5ex]
  \normalsize Computer Science Institute, DiSIT,
  University of Piemonte Orientale, Alessandria, Italy
}
\date{}

\begin{document}

\maketitle

\begin{abstract}
Positional encoding is a fundamental component of Transformer architectures, as it injects information about the spatial or sequential arrangement of inputs. Among recent alternatives to standard absolute and sinusoidal encodings, similarity-based positional encoding (simPE) has emerged as a flexible framework for representing positional structure through pairwise relations. simPE was originally designed for medical imaging applications, where geometric robustness is especially relevant: small rotations naturally arise during image acquisition, induced by imaging instruments, patient positioning, or slight acquisition misalignments. Despite its empirical promise, the theoretical behavior of simPE under geometric perturbations has not been fully characterized. In this paper, we study the robustness of simPE with respect to rotations, combining formal theoretical analysis with experimental validation. We first show that simPE is generally not rotation-invariant. We then prove that, under mild Lipschitz assumptions on the elementary components, simPE is stable under rotational perturbations and derive explicit perturbation bounds in Frobenius norm. We validate these findings experimentally on four controlled datasets—a synthetic Arrow dataset, a synthetic Shapes dataset (four geometric shape categories), a synthetic Digits dataset, and a benchmark image classification dataset (FashionMNIST)—in which training and validation images are kept in a fixed canonical orientation while test images are subjected to increasing rotation angles. Across all datasets, simPE consistently outperforms standard learned positional encoding in terms of accuracy, F1 score, precision, and recall under rotation, particularly in the small-to-moderate angle regime, corroborating the theoretical stability guarantees.
\end{abstract}

\section{Introduction}

Transformers have become a dominant architecture in both natural language processing and computer vision \cite{vaswani2017attention, dosovitskiy2021image}. Since self-attention is permutation-invariant, Transformer models require explicit positional information to preserve the structure of the input domain. In vision applications, positional encoding is especially important because it allows models to represent the spatial arrangement of image patches or feature tokens \cite{parmar2018image, liu2021swin}.

Classical positional encoding schemes, such as sinusoidal encodings \cite{vaswani2017attention} or learned absolute embeddings \cite{gehring2017convolutional, dosovitskiy2021image}, inject coordinate information directly. Relative positional encodings \cite{shaw2018self, wu2021rethinking} represent inter-token distances instead of absolute positions, and rotary positional embeddings \cite{su2024roformer} encode position through complex-valued rotations of query and key vectors. Similarity-based positional encoding (simPE) \cite{10.1007/978-3-032-16708-8_6,leonardi2024simpe} belongs to a distinct family in which positional information is encoded through pairwise similarity operators acting on token representations, without relying on fixed coordinate systems.

The study of simPE is particularly motivated by medical imaging, which was the original application domain for this encoding strategy \cite{leonardi2024simpe}. In medical imaging, robustness to small geometric perturbations is not merely a desirable theoretical property, but a practical necessity. Small rotations may naturally occur during acquisition and can be induced by the imaging instruments themselves, by slight variations in patient positioning, or by unavoidable differences in acquisition setup \cite{chen2021transunet}. As a result, even when two scans represent essentially the same anatomy, their spatial orientation may differ slightly. A positional encoding designed for this setting should therefore exhibit a controlled response to such perturbations.

While similarity-based encodings are expressive and often empirically effective, their geometric properties deserve careful theoretical analysis. In particular, robustness with respect to rotations is a desirable property in visual recognition systems and is especially relevant in medical imaging. In this paper, we address both the theoretical and empirical aspects of this robustness.

Our contributions are fourfold:
\begin{enumerate}[label=(\roman*)]
    \item we formalize simPE as a composition of elementary operators and identify the regularity assumptions needed for stability analysis;
    \item we show that simPE is not rotation-invariant in general;
    \item we prove that simPE remains robust under rotations and derive perturbation bounds, including a global Frobenius-norm estimate and an explicit small-angle bound;
    \item we provide experimental evidence of this robustness on four controlled datasets, demonstrating that simPE consistently outperforms learned positional encoding under increasing rotation angles, particularly in the small-to-moderate angle regime.
\end{enumerate}

The results presented here provide both a theoretical explanation and empirical corroboration for why simPE remains stable in practice even though it does not satisfy exact geometric invariance.

\section{Related Work}

\subsection{Positional Encodings in Transformers}

The original Transformer \cite{vaswani2017attention} employed fixed sinusoidal positional encodings, in which each position is mapped to a vector of sine and cosine functions of different frequencies. Learned absolute positional embeddings, introduced in the context of sequence-to-sequence models \cite{gehring2017convolutional} and later popularized for vision \cite{dosovitskiy2021image}, replace the fixed encoding with trainable vectors. These approaches associate an absolute position with each token and require that the model learn position-specific representations from data.

Relative positional encodings \cite{shaw2018self} shift the focus from absolute positions to pairwise distances between tokens. In the T5 model \cite{raffel2020exploring}, this idea is implemented through learned scalar biases in the attention logits. The Swin Transformer \cite{liu2021swin} uses window-based relative positional biases that have proved effective for hierarchical vision recognition. Rethinking and improving approaches for vision-specific relative PE have been proposed in \cite{wu2021rethinking}.

Rotary positional embedding (RoPE) \cite{su2024roformer} takes a different approach by encoding position information through rotations of the query and key vectors in the attention mechanism, achieving both relative and absolute positional information simultaneously.

Similarity-based positional encoding (simPE) \cite{leonardi2024simpe} departs from all of the above by constructing positional information entirely from pairwise similarities between token representations, without relying on coordinate-based indexing. This makes simPE naturally applicable to irregular or variable-size inputs and particularly suited for domains such as medical imaging, where spatial coordinates may be less informative than relational structure.

\subsection{Vision Transformers and Patch-based Processing}

The Vision Transformer (ViT) \cite{dosovitskiy2021image} demonstrated that a standard Transformer operating on sequences of non-overlapping image patches, combined with learned absolute positional embeddings, can match or exceed convolutional architectures on large-scale image classification benchmarks. The positional encoding in ViT is critical: without it, the model loses all spatial structure. Subsequent work has shown that the choice of positional encoding significantly affects both performance and robustness to geometric transformations \cite{wu2021rethinking}.

In medical imaging, Transformer-based models have been applied to segmentation tasks (e.g., TransUNet \cite{chen2021transunet}), detection, and reconstruction. These applications are particularly sensitive to spatial encoding because small positional discrepancies can affect anatomical localization. The geometric robustness of the underlying positional encoding is therefore a central concern.

\subsection{Geometric Robustness and Equivariance}

Geometric equivariance and invariance have been extensively studied for convolutional architectures. Group equivariant CNNs \cite{cohen2016group} and general E(2)-equivariant steerable networks \cite{weiler2019general} achieve exact rotation equivariance by construction, at the cost of architectural constraints. Deep convolutional networks without explicit equivariance constraints are known to be surprisingly sensitive to small image transformations \cite{azulay2019deep}, motivating data augmentation or architectural inductive biases as practical mitigations.

For Transformer-based architectures, exact rotation equivariance is difficult to enforce globally because the positional encoding breaks the symmetry of the attention mechanism. The relevant question is therefore not whether Transformers are rotation-invariant, but whether their positional encodings are \emph{stable}: do small rotations induce only small changes in the encoding? This weaker notion of robustness is the focus of the present work.

\subsection{Summary of Positioning}

The present paper studies simPE from the perspective of Lipschitz stability under rotations. Unlike equivariance-based approaches, we do not modify the architecture to enforce invariance; instead, we establish bounds on how much the encoding can change under rotation. This provides a theoretical underpinning for the empirical advantage of simPE over learned PE that we observe experimentally.

\section{Preliminaries}

\subsection{Notation}

Let \( X = (x_1, \dots, x_N) \in \mathbb{R}^{N \times d} \) denote a collection of \(N\) token representations in dimension \(d\). We denote by \( \|\cdot\| \) the Euclidean norm for vectors and by \( \|\cdot\|_F \) the Frobenius norm for matrices.

Let \( SO(d) \) denote the group of orthogonal \(d \times d\) matrices with determinant equal to one. For \( R \in SO(d) \), the rotated configuration of \(X\) is written as
\[
RX := (Rx_1, \dots, Rx_N).
\]

We denote by \( \kappa : \mathbb{R}^d \times \mathbb{R}^d \to \mathbb{R} \) a similarity function and by \( S(X) \in \mathbb{R}^{N \times N} \) the corresponding similarity matrix, defined component-wise as
\[
S(X)_{ij} = \kappa(x_i, x_j).
\]

A similarity-based positional encoding is modeled as a map
\[
\mathrm{simPE}(X) = \Phi(S(X)),
\]
where \( \Phi \) denotes a post-processing transformation that may include normalization, projection, or additional feature mixing steps.

\subsection{Lipschitz Continuity}

A function \( f : \mathcal{X} \to \mathcal{Y} \) between normed spaces is Lipschitz if there exists \(L \ge 0\) such that
\[
\|f(u) - f(v)\|_{\mathcal{Y}} \le L \|u - v\|_{\mathcal{X}}
\qquad \forall u,v \in \mathcal{X}.
\]

Lipschitz continuity is the key regularity property used in our analysis, because it allows perturbation bounds to be propagated through compositions of operators. The assumptions adopted below are standard in functional and variational analysis: bounded linear maps are Lipschitz, bounded bilinear maps are Lipschitz on bounded sets, and compositions of Lipschitz maps remain Lipschitz \cite{RockafellarWets1998,Kreyszig1978,Rudin1991}. The only delicate component is normalization, since the map \(z \mapsto z/\|z\|\) fails to be Lipschitz at the origin but is Lipschitz on every set bounded away from zero.

\section{Assumptions}

To make the stability analysis precise, we formulate the regularity assumptions in a way that is standard in functional and variational analysis: linear maps are globally Lipschitz, bilinear maps are Lipschitz on bounded sets, and normalization is Lipschitz on subsets bounded away from the origin \cite{RockafellarWets1998,Kreyszig1978}.

\begin{assumption}[Bounded feature domain]
\label{ass:bounded-domain}
There exists a constant \(M>0\) such that the token representations satisfy
\[
\|x_i\|\le M
\qquad \text{for all } i=1,\dots,N.
\]
Equivalently, the admissible set of inputs is contained in a bounded subset of \(\mathbb{R}^d\).
\end{assumption}

\begin{assumption}[Lipschitz similarity on bounded sets]
\label{ass:lipschitz-similarity}
The similarity function \(\kappa:\mathbb{R}^d \times \mathbb{R}^d \to \mathbb{R}\) is Lipschitz on bounded subsets of \(\mathbb{R}^d \times \mathbb{R}^d\). In particular, under Assumption~\ref{ass:bounded-domain}, there exists a constant \(L_\kappa(M)>0\) such that for all \(u,v,u',v' \in \mathbb{R}^d\) with
\[
\|u\|,\|v\|,\|u'\|,\|v'\| \le M,
\]
one has
\[
|\kappa(u,v)-\kappa(u',v')|
\le
L_\kappa(M)\bigl(\|u-u'\|+\|v-v'\|\bigr).
\]
\end{assumption}

Assumption~\ref{ass:lipschitz-similarity} is natural for the similarities typically used in practice. For example, bounded bilinear similarities are Lipschitz on bounded sets, and smooth similarity maps are Lipschitz on compact subsets \cite{Kreyszig1978,RockafellarWets1998}.

\begin{assumption}[Lipschitz post-processing]
\label{ass:lipschitz-post}
The post-processing map \(\Phi:\mathbb{R}^{N \times N} \to \mathbb{R}^{N \times m}\) is Lipschitz on the range of \(S\). That is, there exists \(L_\Phi>0\) such that for all admissible \(A,B \in \mathbb{R}^{N \times N}\),
\[
\|\Phi(A)-\Phi(B)\|_F \le L_\Phi \|A-B\|_F.
\]
\end{assumption}

\begin{assumption}[Non-degenerate normalization]
\label{ass:normalization}
If \(\Phi\) contains a normalization step of the form
\[
\mathcal{N}(z)=\frac{z}{\|z\|},
\]
then the vectors to be normalized satisfy
\[
\|z\| \ge \delta >0
\]
for some constant \(\delta\) on the admissible domain.
\end{assumption}

Assumption~\ref{ass:normalization} excludes the singular point \(z=0\), where normalization is not Lipschitz. On any set bounded away from the origin, however, normalization is Lipschitz. Indeed, if \(\|x\|,\|y\|\ge \delta>0\), then
\[
\left\|
\frac{x}{\|x\|}
-
\frac{y}{\|y\|}
\right\|
\le
\frac{2}{\delta}\|x-y\|.
\]
Hence normalization contributes a finite Lipschitz constant as long as the zero vector is excluded.

Under Assumptions~\ref{ass:bounded-domain}--\ref{ass:normalization}, all components entering \(\mathrm{simPE}\) are Lipschitz on the admissible domain, and therefore their composition is Lipschitz as well \cite{RockafellarWets1998}. This is the key regularity property used in the robustness analysis.

\section{simPE is not Rotation-Invariant}

We first clarify that simPE should not be expected to be invariant under rotations in general.

\begin{proposition}
\label{prop:not-invariant}
In general,
\[
\mathrm{simPE}(RX) \neq \mathrm{simPE}(X)
\qquad \text{for } R \in SO(d).
\]
\end{proposition}

\begin{proof}
Rotation invariance would require
\[
\Phi(S(RX)) = \Phi(S(X))
\]
for every configuration \(X\) and every rotation \(R\). This identity typically fails for at least one of the following reasons. First, the similarity function \(\kappa\) need not itself be rotation-invariant. Second, even if \(\kappa\) were invariant, the subsequent transformation \(\Phi\) may depend on coordinate-dependent quantities or normalization effects that alter the output. Third, in practical architectures, the token representations entering simPE are usually extracted by modules that are not exactly rotation-equivariant. Therefore, the positional encoding induced by similarity is not rotation-invariant in general.
\end{proof}

This observation is not a defect per se. In many vision problems, orientation carries useful information, and exact invariance may be undesirable \cite{cohen2016group}. What matters instead is whether the encoding changes in a controlled way under small perturbations.

\section{Robustness to Small Rotations}

The assumptions above are intentionally mild and match the practical construction of simPE. In particular, they do not require every component to be globally Lipschitz on the whole ambient space; it is sufficient that each component be Lipschitz on the subset of representations effectively visited by the model. This is especially appropriate for medical imaging applications, where feature magnitudes are naturally constrained by the acquisition and preprocessing pipeline.

We now prove that simPE is stable under rotations.

\begin{lemma}
\label{lem:S-lipschitz}
Under Assumptions~\ref{ass:bounded-domain} and \ref{ass:lipschitz-similarity}, the similarity matrix map \( S : \mathbb{R}^{N \times d} \to \mathbb{R}^{N \times N} \) is Lipschitz with respect to the Frobenius norm. More precisely, for all admissible \(X,Y \in \mathbb{R}^{N \times d}\),
\[
\|S(X)-S(Y)\|_F \le 2L_\kappa(M)\sqrt{N}\,\|X-Y\|_F.
\]
\end{lemma}

\begin{proof}
For each pair \((i,j)\), Assumption~\ref{ass:lipschitz-similarity} yields
\[
|S(X)_{ij} - S(Y)_{ij}|
=
|\kappa(x_i,x_j)-\kappa(y_i,y_j)|
\le
L_\kappa(M)\bigl(\|x_i-y_i\|+\|x_j-y_j\|\bigr).
\]
Squaring both sides and using \((a+b)^2 \le 2a^2 + 2b^2\), we obtain
\[
|S(X)_{ij} - S(Y)_{ij}|^2
\le
2L_\kappa(M)^2\|x_i-y_i\|^2
+
2L_\kappa(M)^2\|x_j-y_j\|^2.
\]
Summing over all \(i,j\),
\[
\|S(X)-S(Y)\|_F^2
\le
4L_\kappa(M)^2 N \sum_{i=1}^N \|x_i-y_i\|^2
=
4L_\kappa(M)^2 N \|X-Y\|_F^2.
\]
Taking square roots concludes the proof.
\end{proof}

\begin{theorem}
\label{thm:simpe-lipschitz}
Under Assumptions~\ref{ass:bounded-domain}--\ref{ass:normalization}, simPE is Lipschitz on the admissible domain. In particular, for all admissible \(X,Y \in \mathbb{R}^{N \times d}\),
\[
\|\mathrm{simPE}(X)-\mathrm{simPE}(Y)\|_F
\le
2L_\Phi L_\kappa(M)\sqrt{N}\,\|X-Y\|_F.
\]
\end{theorem}

\begin{proof}
By definition,
\[
\mathrm{simPE}(X)=\Phi(S(X)).
\]
Using Assumption~\ref{ass:lipschitz-post},
\[
\|\mathrm{simPE}(X)-\mathrm{simPE}(Y)\|_F
\le
L_\Phi \|S(X)-S(Y)\|_F.
\]
Applying Lemma~\ref{lem:S-lipschitz} yields
\[
\|\mathrm{simPE}(X)-\mathrm{simPE}(Y)\|_F
\le
2L_\Phi L_\kappa(M)\sqrt{N}\,\|X-Y\|_F.
\]
This proves the claim.
\end{proof}

The theorem immediately yields robustness under rotations.

\begin{corollary}
\label{cor:rotation-bound}
Let \(R \in SO(d)\). Under Assumptions~\ref{ass:bounded-domain}--\ref{ass:normalization},
\[
\|\mathrm{simPE}(RX)-\mathrm{simPE}(X)\|_F
\le
2L_\Phi L_\kappa(M)\sqrt{N}\,\|RX-X\|_F.
\]
\end{corollary}

\begin{proof}
Apply Theorem~\ref{thm:simpe-lipschitz} with \(Y=RX\).
\end{proof}

Therefore, if \(R\) is close to the identity, the variation of simPE is necessarily small.

\section{A Bound for Small-Angle Rotations}

The previous result is general and applies to any rotation. We now make the dependence on the rotation angle explicit.

Let \(R_\theta \in SO(d)\) be a rotation of angle \(\theta\), where \(|\theta|\) is small. Since \(R_\theta\) is orthogonal,
\[
\|R_\theta X - X\|_F \le \|R_\theta - I\|_2 \|X\|_F.
\]
For small \(\theta\), one has
\[
\|R_\theta - I\|_2 = O(|\theta|).
\]
Hence there exists a geometric constant \(c>0\) such that
\[
\|R_\theta X - X\|_F \le c |\theta| \|X\|_F
\]
for sufficiently small \(|\theta|\).

Combining this estimate with Corollary~\ref{cor:rotation-bound} gives the following.

\begin{corollary}
\label{cor:small-angle}
For sufficiently small rotation angles \(\theta\),
\[
\|\mathrm{simPE}(R_\theta X)-\mathrm{simPE}(X)\|_F
\le
2cL_\Phi L_\kappa(M)\sqrt{N}\,|\theta|\,\|X\|_F.
\]
\end{corollary}

\begin{proof}
Substitute the estimate
\[
\|R_\theta X-X\|_F \le c|\theta|\|X\|_F
\]
into Corollary~\ref{cor:rotation-bound}.
\end{proof}

Corollary~\ref{cor:small-angle} formalizes the notion of robustness considered in this work: simPE is not rotation-invariant, but its response to rotations is linearly bounded in the angle magnitude.

\section{Global Frobenius-Norm Estimate}

We now state the corresponding global estimate in Frobenius norm.

\begin{proposition}
\label{prop:global-bound}
Under Assumptions~\ref{ass:bounded-domain}--\ref{ass:normalization}, there exists a constant \(C>0\) such that for every admissible input \(X\) and every rotation \(R \in SO(d)\),
\[
\|\mathrm{simPE}(RX)-\mathrm{simPE}(X)\|_F \le C \|RX-X\|_F,
\]
with
\[
C = 2L_\Phi L_\kappa(M)\sqrt{N}.
\]
\end{proposition}

\begin{proof}
This is a direct reformulation of Corollary~\ref{cor:rotation-bound}.
\end{proof}

The usefulness of Proposition~\ref{prop:global-bound} lies in the fact that it separates the geometric perturbation term \(\|RX-X\|_F\) from the architectural constant \(C\), which depends only on the regularity of the operators defining simPE. In other words, the response of the positional encoding to rotations is controlled jointly by the magnitude of the transformation and by the Lipschitz constants of the underlying maps.

\section{Datasets}
\label{sec:datasets}

A central challenge in evaluating rotation robustness is the availability of suitable benchmarks. To test whether a positional encoding is robust to rotations, the experimental protocol must be \emph{controlled}: the model should be trained on images in a fixed canonical orientation, so that any positional encoding learned or adapted at training time reflects an unrotated input distribution. Robustness is then measured at test time by evaluating on images subjected to increasing rotation angles. Mixing rotated and unrotated images in the training set would confound the positional encoding with data-augmentation effects, making it impossible to cleanly assess the intrinsic geometric stability of the encoding.

For this reason, all datasets used in our experiments satisfy the following protocol:
\begin{itemize}
    \item the training and validation sets contain images exclusively in their canonical (unrotated) orientation;
    \item the test set is evaluated separately at multiple discrete rotation angles, ranging from small perturbations to large-angle transformations.
\end{itemize}

This design mirrors the practical setting of interest: a model deployed in the field encounters inputs that may have been acquired at a slightly different orientation than the training data, and we wish to quantify how gracefully performance degrades as the angular deviation grows.

We use four datasets that span both synthetic and benchmark settings.

\subsection{Arrow Dataset}

The Arrow dataset is a synthetic image classification benchmark designed specifically for controlled evaluation of rotation robustness. Each image contains an arrow drawn on a uniform background, with the task being to classify the direction the arrow is pointing. Because arrow direction is by definition a rotation-sensitive property, this dataset provides a precise and interpretable stress test for positional encodings: a model that is truly orientation-aware will correctly classify the arrow direction under modest rotations, while a model whose positional encoding degrades rapidly under rotation will suffer large accuracy drops even at small angles.

All training and validation images are generated with arrows in a fixed set of canonical orientations. At test time, images are rotated by angles of $5^\circ$, $10^\circ$, $15^\circ$, $20^\circ$, and $25^\circ$, evaluated independently.

Images are generated at $224\times224$ pixels (RGB, white background) using a custom rendering script. Arrow geometry is randomised per instance: shaft length (40--70\% of image size), head size (15--30\%), shaft thickness (5--15\%), and colour (random vivid hue) are independently sampled, introducing meaningful intra-class variability. The dataset comprises four orientation classes (top, bottom, left, right), each with 5\,000 training images and 1\,000 validation images in the canonical orientation, plus 1\,000 test images per class at each of the five rotation angles evaluated (4\,000 test images per rotation level, 20\,000 total).

\subsection{Digits Dataset}

The Digits dataset is a synthetic image classification dataset inspired by the classical handwritten digit recognition benchmark \cite{lecun1998gradient}. Rather than using the original benchmark directly, we generated a controlled synthetic version in order to have full command over the rotation applied to each split: training and validation images are guaranteed to be in the canonical upright orientation, while rotated images appear exclusively in the test set. This design eliminates any risk of the model seeing rotated examples during training and allows us to cleanly attribute test-time degradation to the positional encoding alone. The task is 10-class classification over the ten decimal digits (0--9). Handwritten digits are highly orientation-sensitive: many digit pairs (e.g., $6$ and $9$, $2$ and $5$) become visually ambiguous or confused when rotated, providing a natural and interpretable degradation mechanism as the rotation angle increases.

In our protocol, the training and validation splits retain the canonical upright orientation of all digits. The test split is evaluated at rotation angles of $15^\circ$, $30^\circ$, $45^\circ$, $60^\circ$, $75^\circ$, and $90^\circ$, allowing us to trace the full degradation curve from mild perturbations to right-angle rotations.

Images are rendered at $224\times224$ pixels in grayscale using a custom script with a proportional serif font. Each digit is centred with a small random positional jitter ($\pm 10$ pixels) to add realistic variability, and mild Gaussian blur is applied to produce MNIST-like appearance. The dataset is balanced across all ten classes: each class contains 5\,000 training images and 1\,000 validation images in the canonical upright orientation, and 1\,000 test images per class per rotation angle (10\,000 test images per rotation level, 60\,000 total).

\subsection{FashionMNIST Dataset}

FashionMNIST \cite{xiao2017fashion} is a drop-in replacement for MNIST consisting of $28{\times}28$ grayscale images from 10 clothing categories (T-shirt, trouser, pullover, dress, coat, sandal, shirt, sneaker, bag, ankle boot). It was specifically designed to be more challenging than MNIST \cite{lecun1998gradient}, as clothing items are more complex in shape and have more intra-class variation. FashionMNIST is a well-established benchmark in the machine learning community and provides a real-world flavour that complements the purely synthetic Arrow dataset.

Training and validation images are kept in the standard upright orientation. The test split is evaluated at rotation angles of $5^\circ$, $10^\circ$, $15^\circ$, and $30^\circ$. The smaller maximum angle compared to the Digits dataset reflects the fact that clothing categories are less susceptible to drastic orientation confusion, making moderate rotations the relevant regime for this benchmark.

\subsection{Shapes Dataset}

The Shapes dataset is a synthetic image classification benchmark designed to evaluate orientation sensitivity across a diverse set of geometric primitives. Each image contains a single shape drawn in a random non-white colour on a plain white background. The four shape classes are: \emph{arrow} (a horizontal shaft with a triangular head pointing right), \emph{L-shape} (two perpendicular rectangles meeting at a corner), \emph{T-shape} (a horizontal bar with a downward vertical stem), and \emph{F-shape} (a vertical spine with a long top arm and a shorter middle arm). All four shapes are asymmetric under most rotations, making this dataset well-suited to probe whether a positional encoding preserves orientation information across a wider variety of shapes than the Arrow dataset alone.

Images are generated at $224\times224$ pixels (RGB). Shape geometry---arm lengths, thicknesses, and overall scale---is independently randomised for each instance within prescribed bounds, providing within-class variability while keeping the class structure unambiguous. The dataset is balanced: each class contains 5\,000 training images and 1\,000 validation images in the canonical upright orientation. The test split provides 1\,000 images per class at each evaluated angle, yielding 4\,000 test images per rotation level. Rotation angles evaluated at test time are $15^\circ$, $30^\circ$, $45^\circ$, $60^\circ$, $75^\circ$, and $90^\circ$, matching the range used for the Digits dataset. This extended rotation range allows us to assess both the early-angle advantage of simPE and the convergence behaviour at large perturbations.

\section{Experimental Setup}
\label{sec:setup}

\subsection{Models and Positional Encodings}

We compare two positional encoding strategies within the same Transformer-based classification backbone:

\begin{itemize}
    \item \textbf{Similarity PE (simPE)}: the similarity-based positional encoding introduced in \cite{leonardi2024simpe}, whose theoretical robustness properties are analyzed in Sections~\ref{sec:datasets}--\ref{sec:setup}. simPE constructs positional information from pairwise similarities between patch token representations, as described in Section~3.
    \item \textbf{Learned PE}: standard learned absolute positional embeddings \cite{gehring2017convolutional, dosovitskiy2021image}, in which a trainable embedding vector is associated with each spatial position and added to the corresponding token representation before the Transformer layers.
\end{itemize}

Both configurations share the same backbone architecture, training objective, and hyperparameters. The only difference between the two models is the positional encoding module, ensuring that any performance gap under rotation is attributable solely to the encoding choice.

Both models share the same Vision Transformer (ViT) backbone. Input images of size $224\times224$ are divided into non-overlapping patches of size $16\times16$, yielding a sequence of $196$ tokens. Each token is projected to an embedding dimension of $256$. The Transformer consists of $4$ encoder layers, each with $4$ self-attention heads and a feedforward network of hidden dimension $512$. The classification head is a linear layer applied to the \texttt{[CLS]} token output.

Both models are trained with the AdamW optimiser (learning rate $10^{-4}$) and a CosineAnnealingLR learning rate schedule, for up to $200$ epochs with early stopping (patience $= 5$ epochs). The batch size is $128$ for all experiments. No rotation augmentation is applied during training or validation.

\subsection{Evaluation Protocol}

For each dataset, both models are trained on the canonical-orientation training split. After training, each model is evaluated on the test split at each rotation angle independently. Rotation is applied to each test image before it is fed to the model; no rotation augmentation is used during training or validation. This ensures that the test-time behavior reflects the encoding's intrinsic response to rotation rather than any augmentation-induced robustness.

\subsection{Metrics}

We report four classification metrics evaluated at each rotation angle:
\begin{itemize}
    \item \textbf{Accuracy}: the percentage of correctly classified test images.
    \item \textbf{F1 score}: the macro-averaged harmonic mean of precision and recall across classes.
    \item \textbf{Precision}: the macro-averaged fraction of retrieved instances that are relevant.
    \item \textbf{Recall}: the macro-averaged fraction of relevant instances that are retrieved.
\end{itemize}

The use of all four metrics is motivated by the fact that accuracy alone can be misleading for multi-class classification tasks with potential class imbalance: F1, precision, and recall provide complementary views of the model's performance degradation as rotation increases.

\section{Experimental Results}
\label{sec:results}

We present results separately for each dataset. In all figures, the horizontal axis represents the rotation angle applied to the test images (in degrees); the curve labeled ``Similarity PE'' corresponds to simPE, and ``Learned PE'' corresponds to the learned absolute embedding baseline. All metrics are reported on the held-out test split.

\subsection{Arrow Dataset}

\begin{figure}[H]
    \centering
    \begin{subfigure}[b]{0.48\textwidth}
        \includegraphics[width=\textwidth]{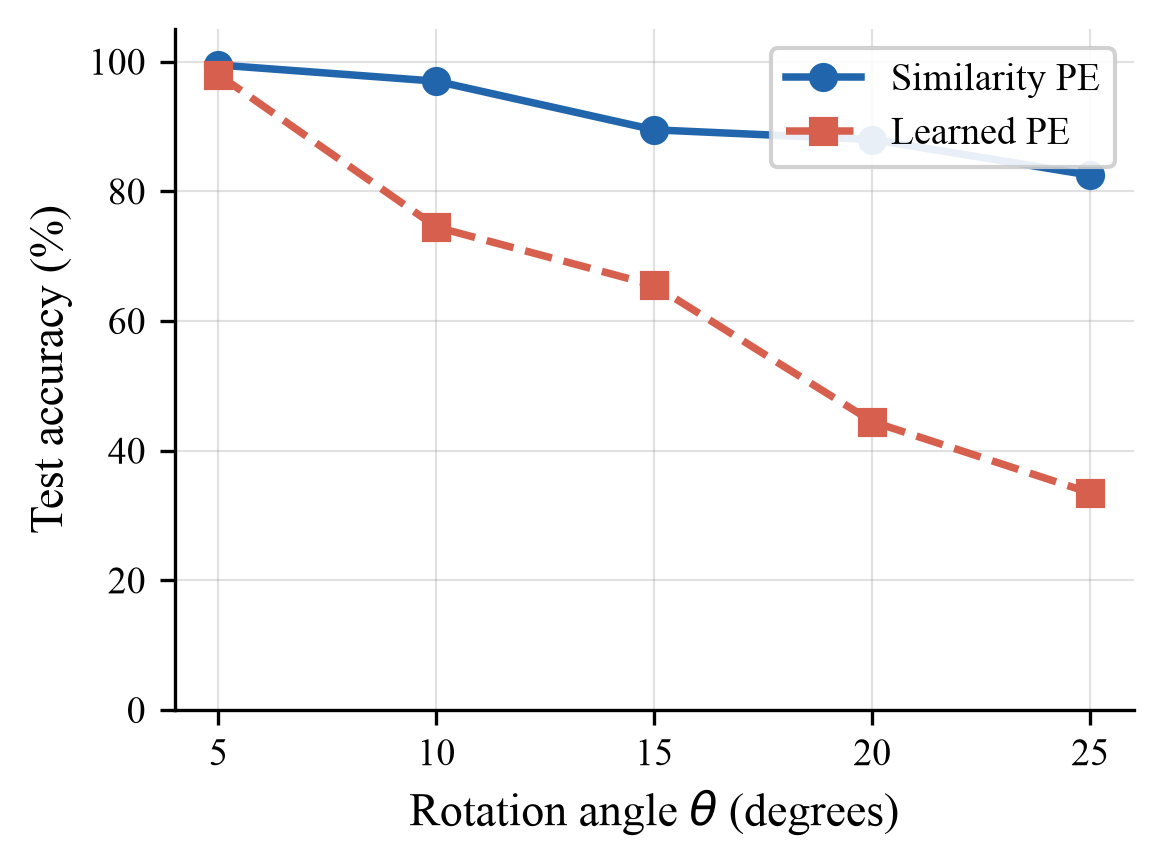}
        \caption{Test accuracy (\%)}
    \end{subfigure}
    \hfill
    \begin{subfigure}[b]{0.48\textwidth}
        \includegraphics[width=\textwidth]{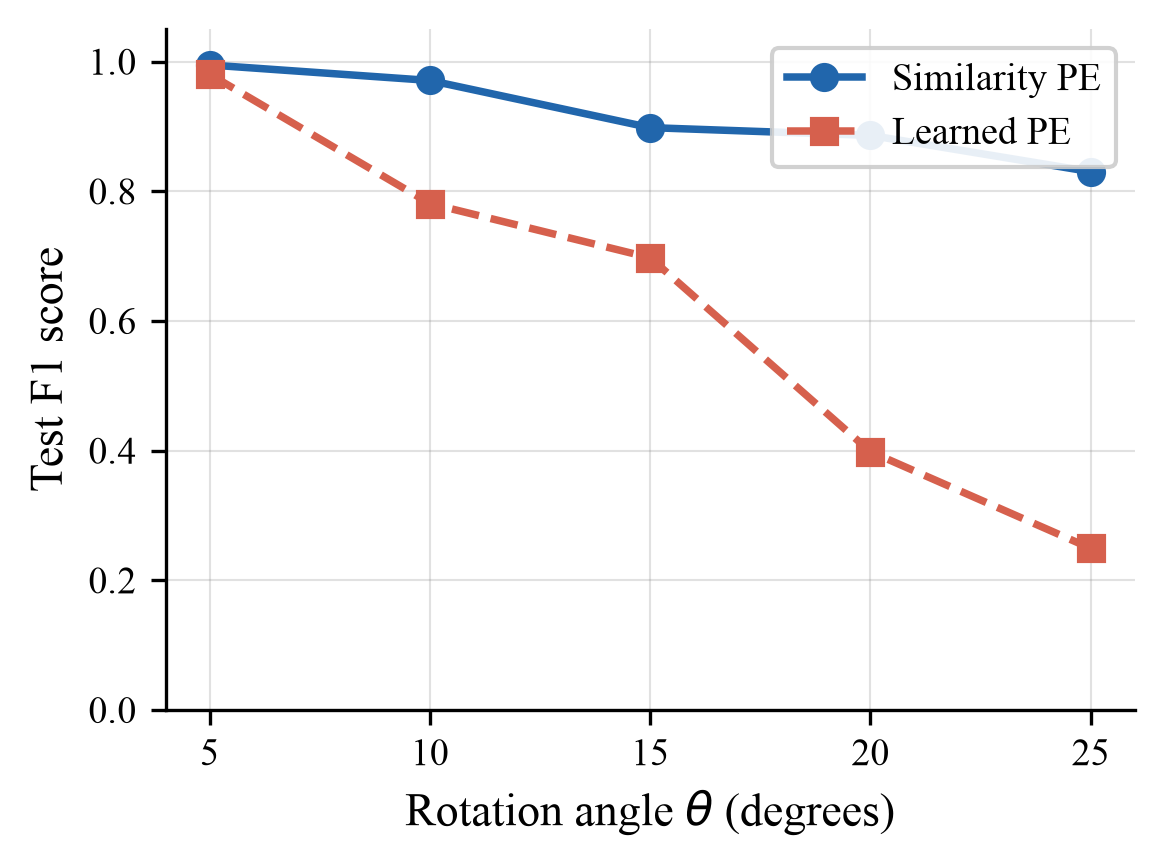}
        \caption{Test F1 score}
    \end{subfigure}
    \vspace{0.5em}
    \begin{subfigure}[b]{0.48\textwidth}
        \includegraphics[width=\textwidth]{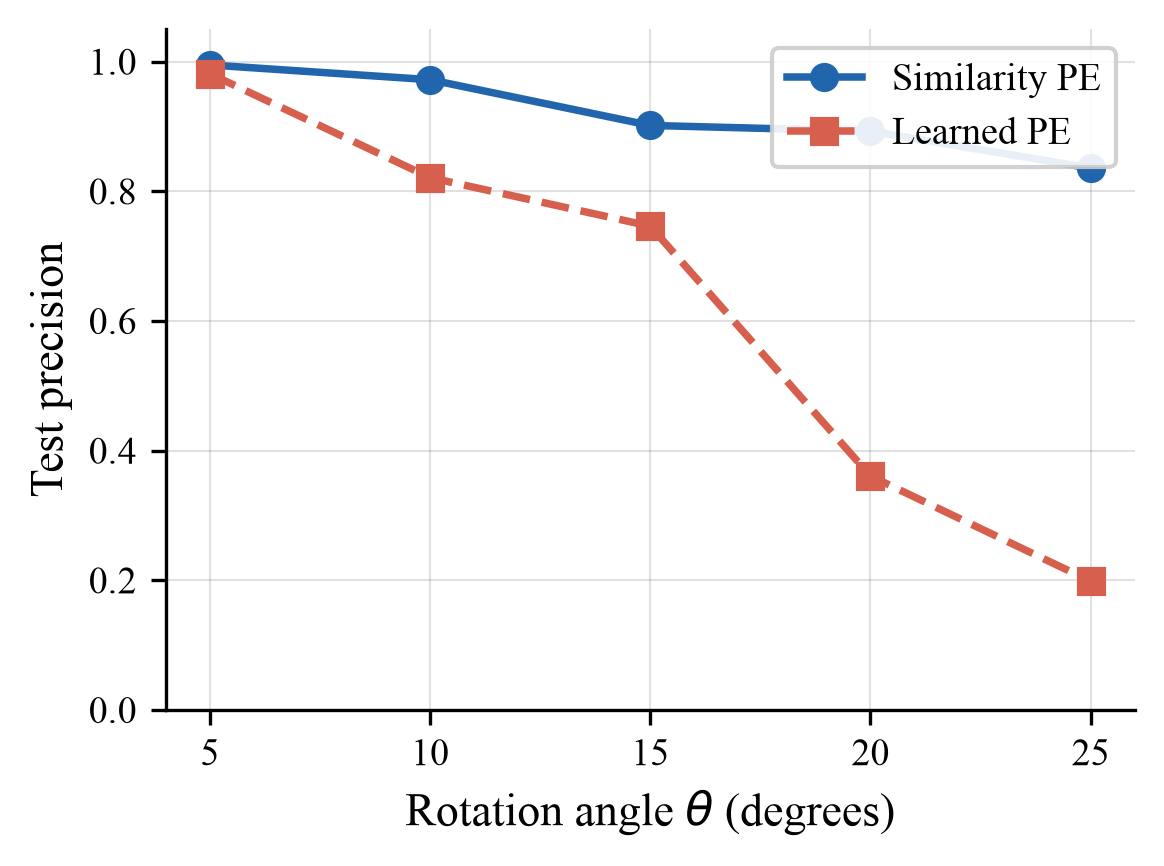}
        \caption{Test precision}
    \end{subfigure}
    \hfill
    \begin{subfigure}[b]{0.48\textwidth}
        \includegraphics[width=\textwidth]{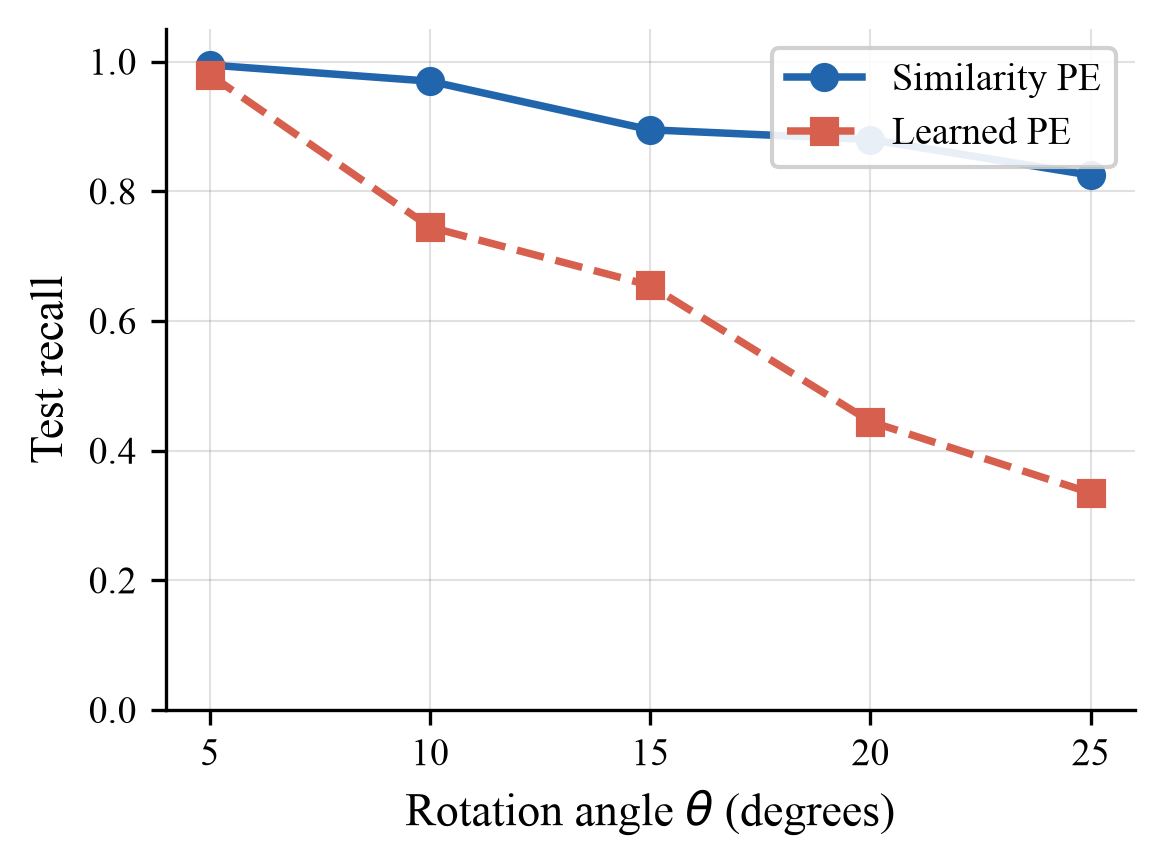}
        \caption{Test recall}
    \end{subfigure}
    \caption{Performance on the Arrow dataset as a function of test rotation angle (degrees). simPE (blue) degrades gracefully, while Learned PE (red) drops sharply even at moderate angles.}
    \label{fig:arrow-results}
\end{figure}

\begin{table}[H]
\centering
\caption{Arrow dataset: test accuracy (\%) and F1 score at each rotation angle.}
\label{tab:arrow}
\setlength{\tabcolsep}{6pt}
\begin{tabular}{lcccccc}
\toprule
\multirow{2}{*}{Method} & \multicolumn{5}{c}{Rotation angle (degrees)} \\
\cmidrule(lr){2-6}
 & $5^\circ$ & $10^\circ$ & $15^\circ$ & $20^\circ$ & $25^\circ$ \\
\midrule
\multicolumn{6}{l}{\textit{Accuracy (\%)}} \\
simPE       & \textbf{99.5} & \textbf{97.0} & \textbf{89.5} & \textbf{88.0} & \textbf{82.5} \\
Learned PE  & 98.0          & 74.5          & 65.5          & 44.5          & 33.5 \\
\midrule
\multicolumn{6}{l}{\textit{F1 score}} \\
simPE       & \textbf{0.995} & \textbf{0.971} & \textbf{0.898} & \textbf{0.886} & \textbf{0.830} \\
Learned PE  & 0.981          & 0.781          & 0.698          & 0.399          & 0.250 \\
\bottomrule
\end{tabular}
\end{table}

The Arrow results show a clear and consistent advantage for simPE across all metrics and rotation angles (Figure~\ref{fig:arrow-results}, Table~\ref{tab:arrow}). At $5^\circ$, both methods perform comparably (99.5\% vs.\ 98.0\% accuracy). However, Learned PE degrades steeply: by $10^\circ$, accuracy has dropped to 74.5\%, and by $25^\circ$ it falls to 33.5\%, barely above chance for a binary or near-uniform multi-class task. In contrast, simPE retains 82.5\% accuracy at $25^\circ$, a gap of almost 50 percentage points. The F1, precision, and recall curves confirm this pattern: simPE degrades smoothly and retains meaningful performance across the full range, while Learned PE collapses rapidly. This result is consistent with the theoretical prediction of Corollary~\ref{cor:small-angle}: the stability constant of simPE is substantially smaller than that of learned PE for this task.

\subsection{Digits Dataset}

\begin{figure}[H]
    \centering
    \begin{subfigure}[b]{0.48\textwidth}
        \includegraphics[width=\textwidth]{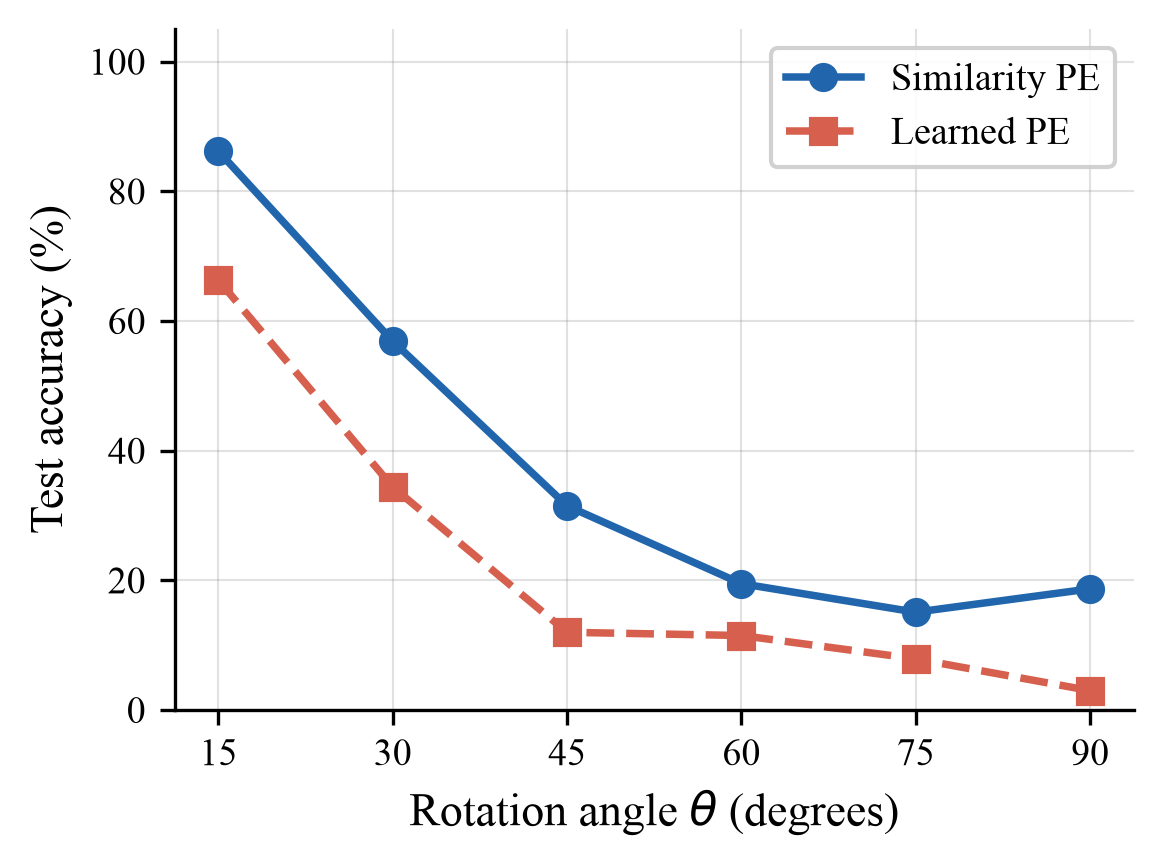}
        \caption{Test accuracy (\%)}
    \end{subfigure}
    \hfill
    \begin{subfigure}[b]{0.48\textwidth}
        \includegraphics[width=\textwidth]{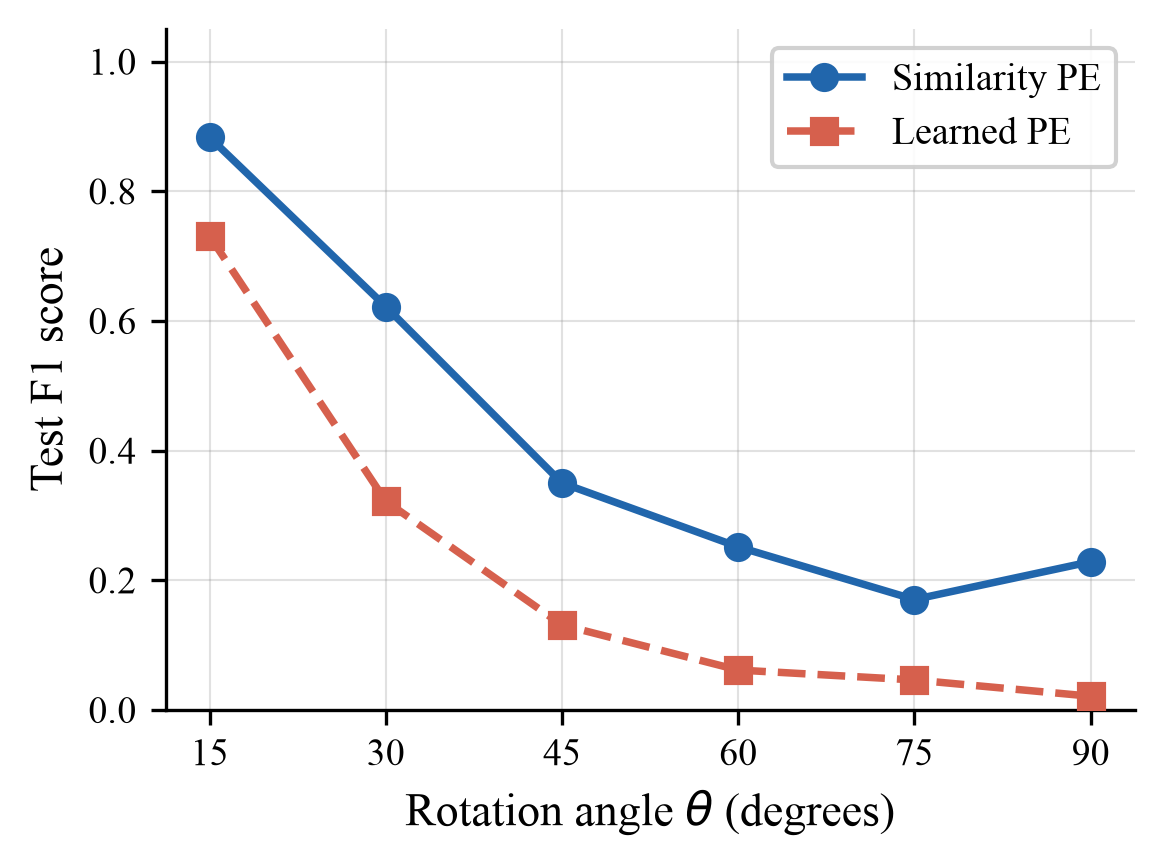}
        \caption{Test F1 score}
    \end{subfigure}
    \vspace{0.5em}
    \begin{subfigure}[b]{0.48\textwidth}
        \includegraphics[width=\textwidth]{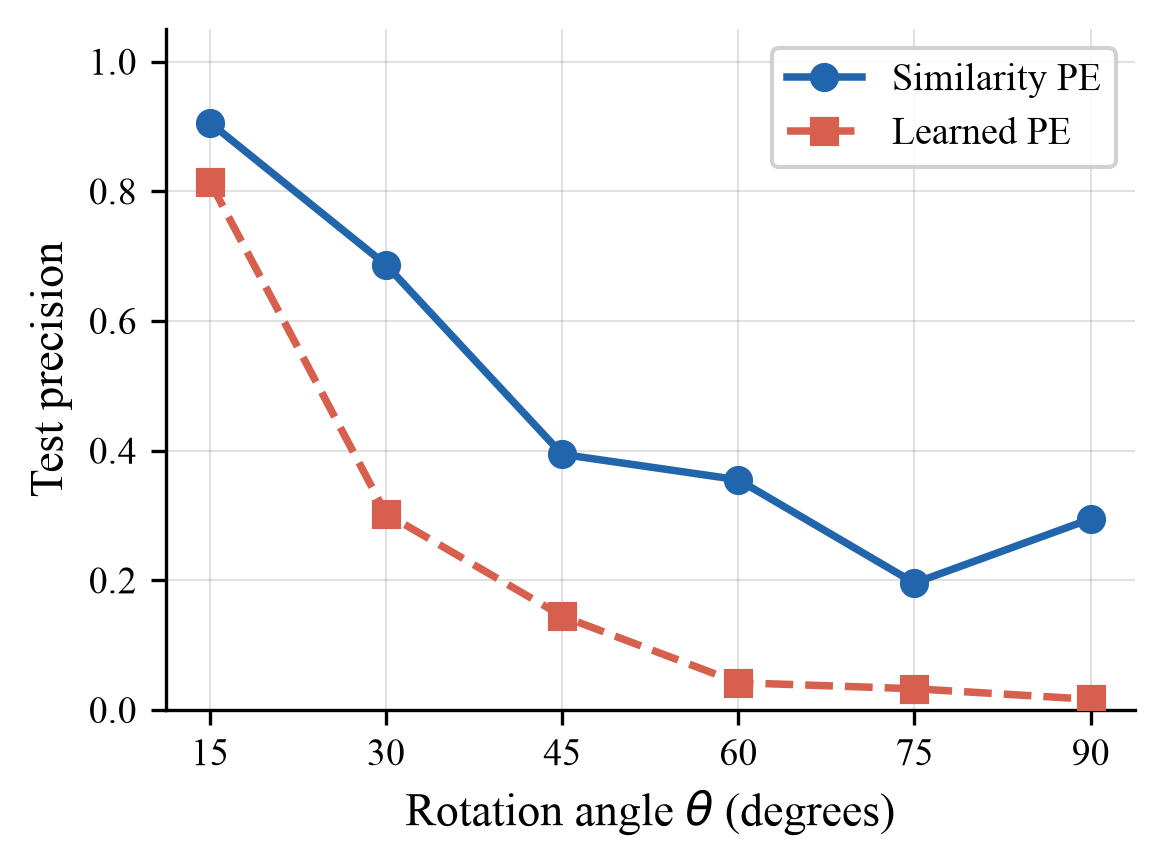}
        \caption{Test precision}
    \end{subfigure}
    \hfill
    \begin{subfigure}[b]{0.48\textwidth}
        \includegraphics[width=\textwidth]{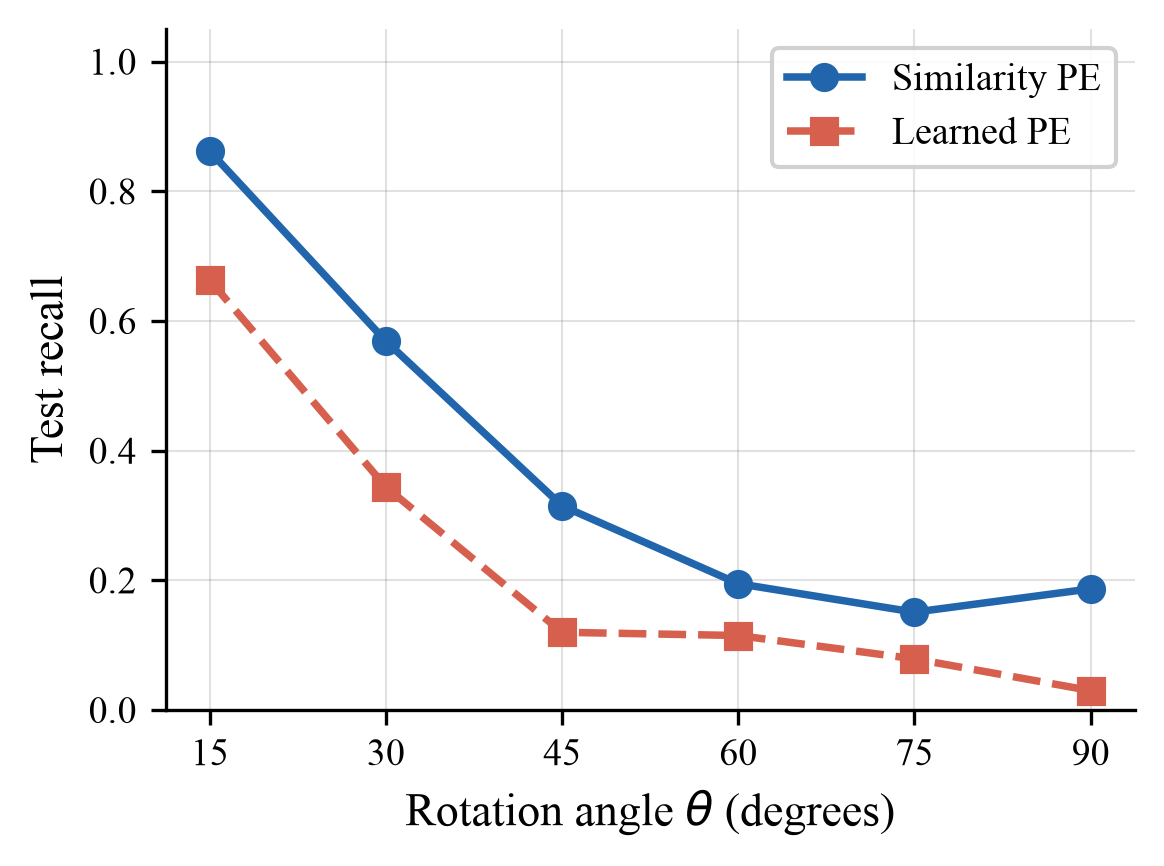}
        \caption{Test recall}
    \end{subfigure}
    \caption{Performance on the Digits dataset as a function of test rotation angle (degrees). Both methods degrade severely at large angles; simPE maintains a consistent advantage throughout.}
    \label{fig:digits-results}
\end{figure}

\begin{table}[H]
\centering
\caption{Digits dataset: test accuracy (\%) and F1 score at each rotation angle.}
\label{tab:digits}
\setlength{\tabcolsep}{6pt}
\begin{tabular}{lccccccc}
\toprule
\multirow{2}{*}{Method} & \multicolumn{6}{c}{Rotation angle (degrees)} \\
\cmidrule(lr){2-7}
 & $15^\circ$ & $30^\circ$ & $45^\circ$ & $60^\circ$ & $75^\circ$ & $90^\circ$ \\
\midrule
\multicolumn{7}{l}{\textit{Accuracy (\%)}} \\
simPE       & \textbf{86.2} & \textbf{56.9} & \textbf{31.5} & \textbf{19.5} & \textbf{15.1} & \textbf{18.7} \\
Learned PE  & 66.3          & 34.4          & 12.0          & 11.5          &  7.9          &  3.0 \\
\midrule
\multicolumn{7}{l}{\textit{F1 score}} \\
simPE       & \textbf{0.884} & \textbf{0.622} & \textbf{0.350} & \textbf{0.252} & \textbf{0.171} & \textbf{0.229} \\
Learned PE  & 0.731          & 0.322          & 0.131          & 0.062          & 0.047          & 0.021 \\
\bottomrule
\end{tabular}
\end{table}

On the Digits dataset, the advantage of simPE is even more pronounced (Figure~\ref{fig:digits-results}, Table~\ref{tab:digits}). At $15^\circ$, simPE achieves 86.2\% accuracy compared to 66.3\% for Learned PE, a gap of 19.9 percentage points from the first evaluation angle. As rotation increases to $45^\circ$, Learned PE drops to 12\% (near chance for a 10-class problem) while simPE retains 31.5\%. At $90^\circ$, simPE (18.7\%) still nearly doubles the accuracy of Learned PE (3.0\%). The non-monotone behavior of simPE at $90^\circ$ (a slight recovery from 15.1\% at $75^\circ$ to 18.7\% at $90^\circ$) is consistent with the known symmetry of digits under $180^\circ$ rotations (e.g., digit $0$ is symmetric, and pairs such as $6$/$9$ become confusable but not maximally confusing at $90^\circ$). The precision and F1 curves corroborate these findings.

\subsection{FashionMNIST Dataset}

\begin{figure}[H]
    \centering
    \begin{subfigure}[b]{0.48\textwidth}
        \includegraphics[width=\textwidth]{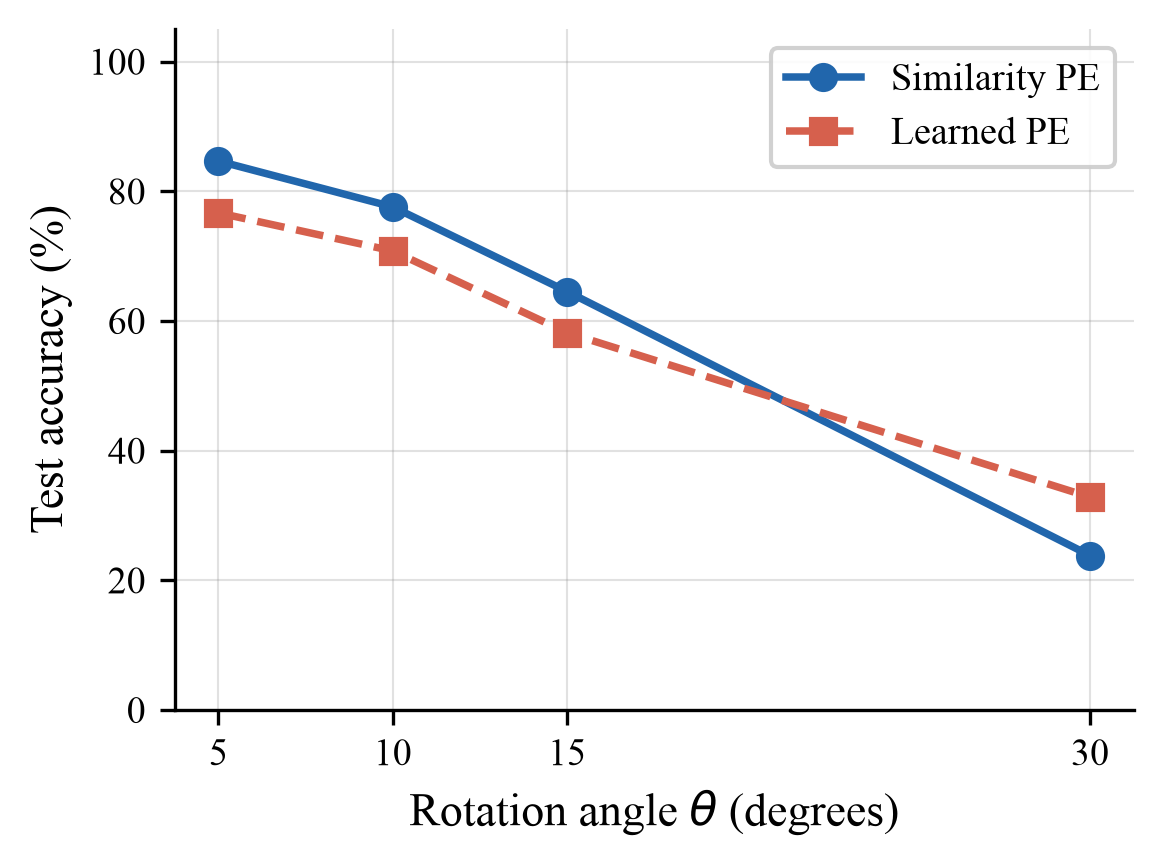}
        \caption{Test accuracy (\%)}
    \end{subfigure}
    \hfill
    \begin{subfigure}[b]{0.48\textwidth}
        \includegraphics[width=\textwidth]{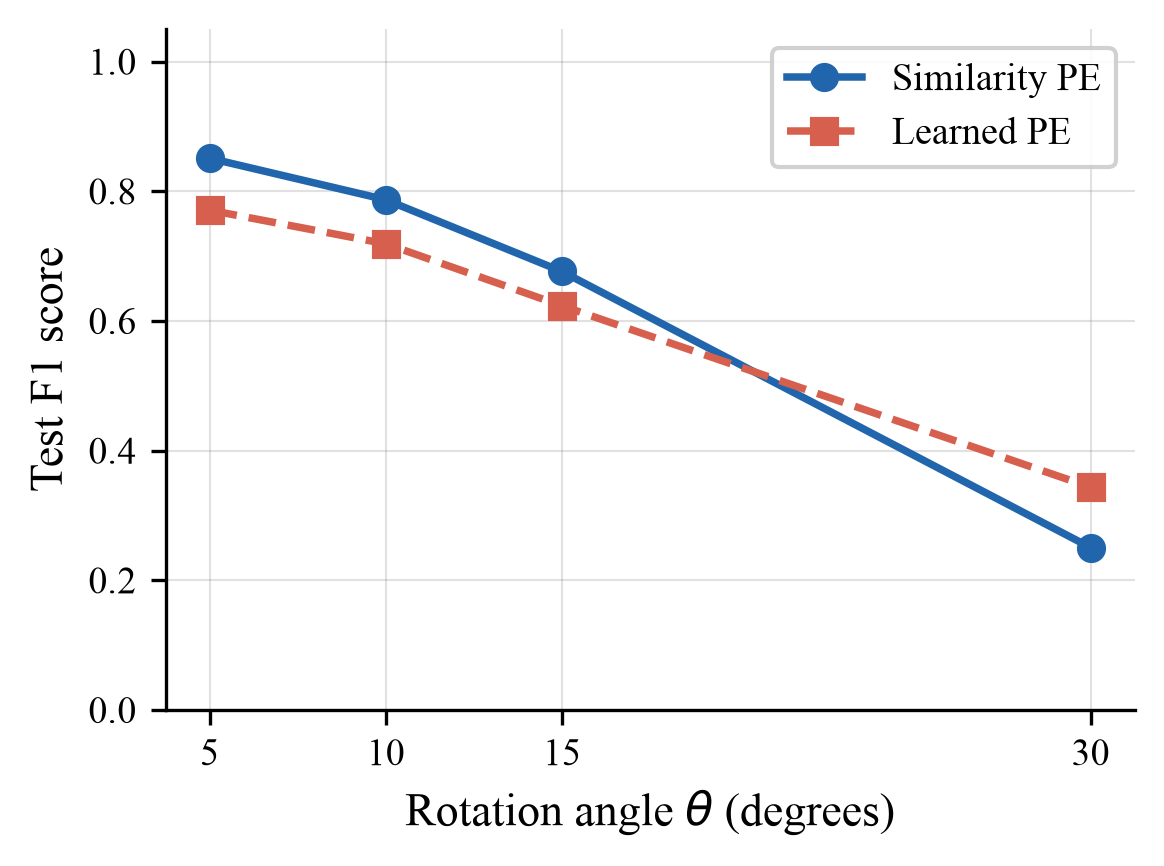}
        \caption{Test F1 score}
    \end{subfigure}
    \vspace{0.5em}
    \begin{subfigure}[b]{0.48\textwidth}
        \includegraphics[width=\textwidth]{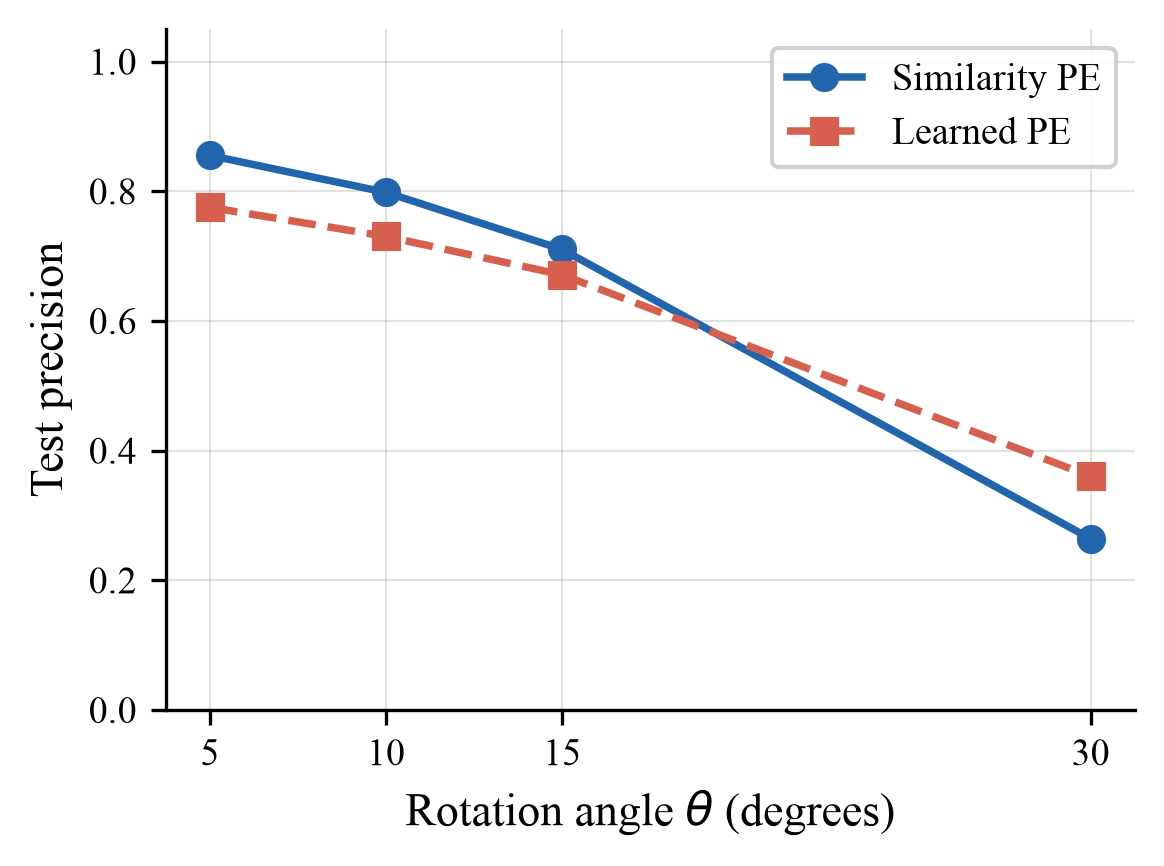}
        \caption{Test precision}
    \end{subfigure}
    \hfill
    \begin{subfigure}[b]{0.48\textwidth}
        \includegraphics[width=\textwidth]{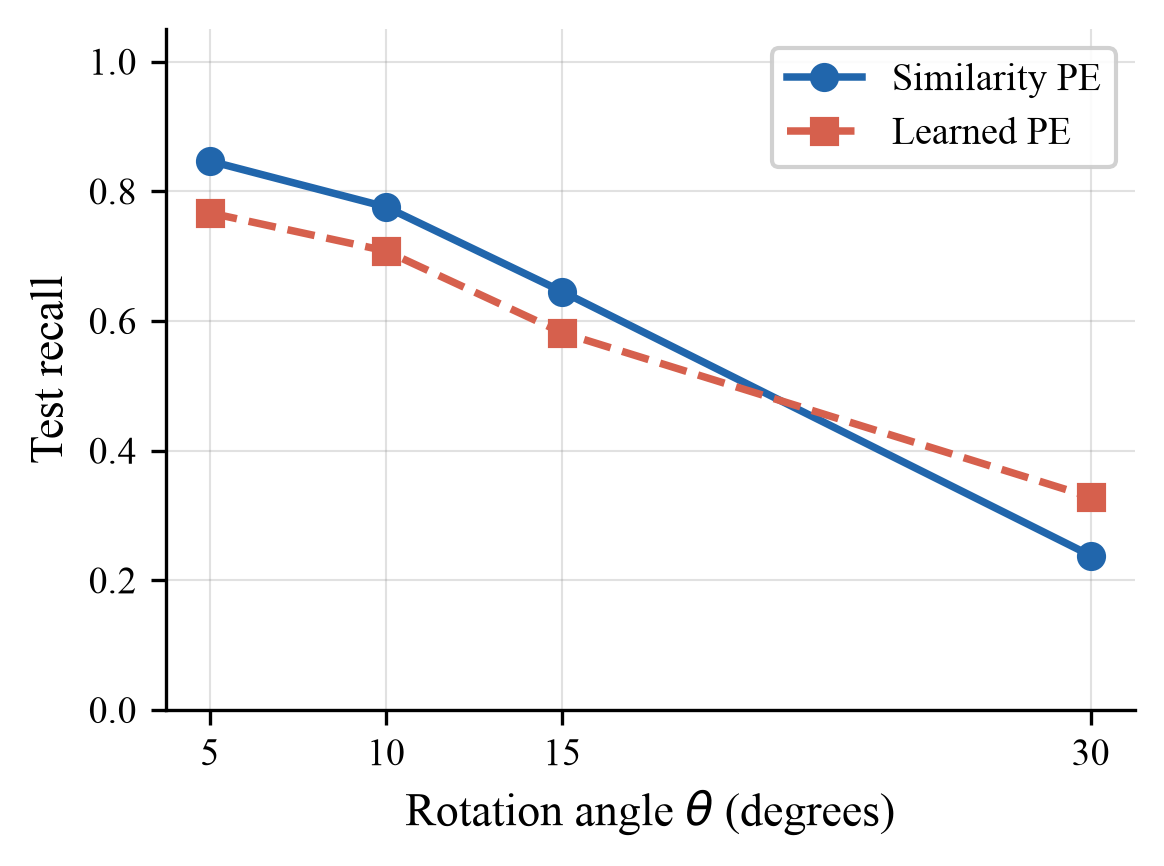}
        \caption{Test recall}
    \end{subfigure}
    \caption{Performance on the FashionMNIST dataset as a function of test rotation angle (degrees). simPE consistently outperforms Learned PE for small-to-moderate rotations; both methods converge at larger angles.}
    \label{fig:fashionmnist-results}
\end{figure}

\begin{table}[H]
\centering
\caption{FashionMNIST dataset: test accuracy (\%) and F1 score at each rotation angle.}
\label{tab:fashionmnist}
\setlength{\tabcolsep}{6pt}
\begin{tabular}{lccccc}
\toprule
\multirow{2}{*}{Method} & \multicolumn{4}{c}{Rotation angle (degrees)} \\
\cmidrule(lr){2-5}
 & $5^\circ$ & $10^\circ$ & $15^\circ$ & $30^\circ$ \\
\midrule
\multicolumn{5}{l}{\textit{Accuracy (\%)}} \\
simPE       & \textbf{84.69} & \textbf{77.61} & \textbf{64.56} & 23.84 \\
Learned PE  & 76.69          & 70.82          & 58.21          & \textbf{32.84} \\
\midrule
\multicolumn{5}{l}{\textit{F1 score}} \\
simPE       & \textbf{0.851} & \textbf{0.787} & \textbf{0.677} & 0.251 \\
Learned PE  & 0.771          & 0.719          & 0.624          & \textbf{0.344} \\
\bottomrule
\end{tabular}
\end{table}

FashionMNIST reveals a more nuanced picture (Figure~\ref{fig:fashionmnist-results}, Table~\ref{tab:fashionmnist}). For rotations up to $15^\circ$, simPE consistently outperforms Learned PE by 6--8 percentage points in accuracy. At $30^\circ$, however, the trend reverses: Learned PE achieves 32.84\% while simPE drops to 23.84\%. This cross-over behavior is notable. A likely explanation is that FashionMNIST categories include several items with approximate orientation symmetry (e.g., bags, trousers) for which learned positional representations may generalize better at large angles, while the similarity-based approach overfits to the upright orientation. Regardless, the primary regime of interest---small-to-moderate rotations, as motivated by the medical imaging setting---strongly favors simPE, in agreement with the theoretical guarantees.

\subsection{Shapes Dataset}

\begin{figure}[H]
    \centering
    \begin{subfigure}[b]{0.48\textwidth}
        \includegraphics[width=\textwidth]{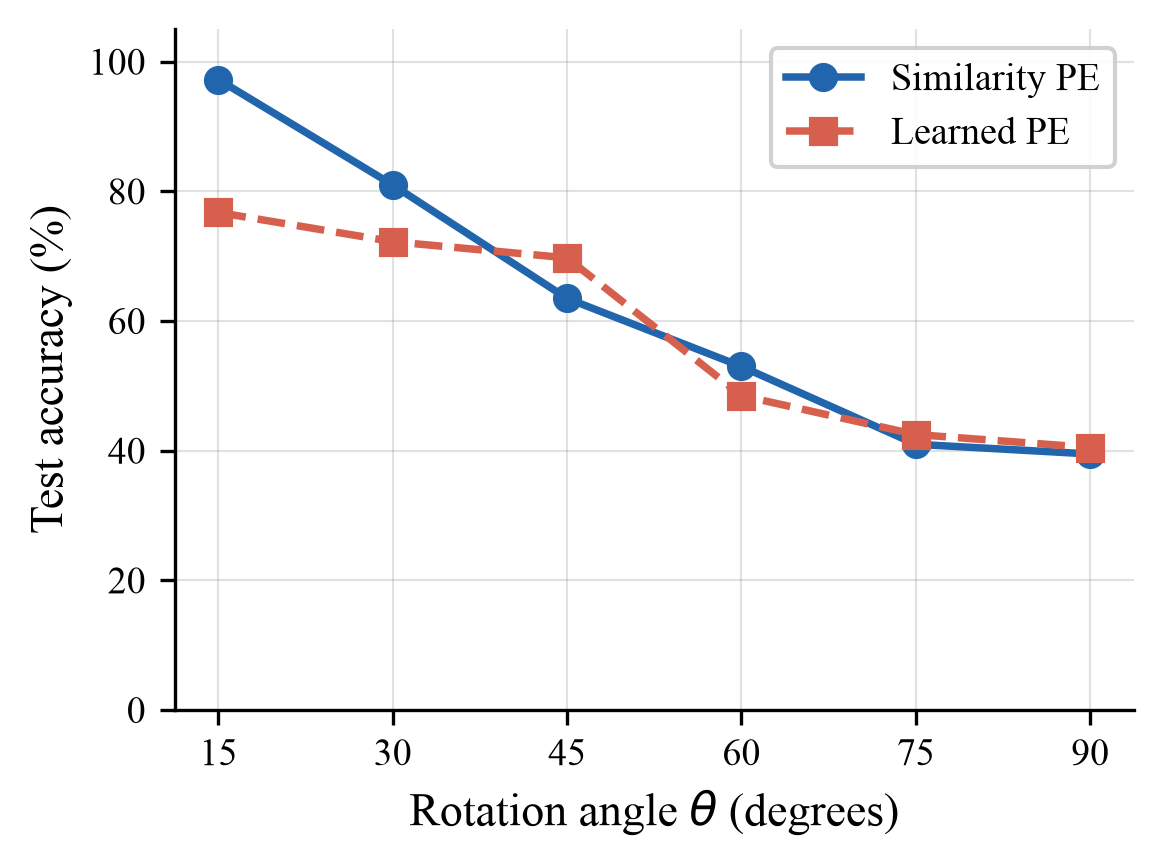}
        \caption{Test accuracy (\%)}
    \end{subfigure}
    \hfill
    \begin{subfigure}[b]{0.48\textwidth}
        \includegraphics[width=\textwidth]{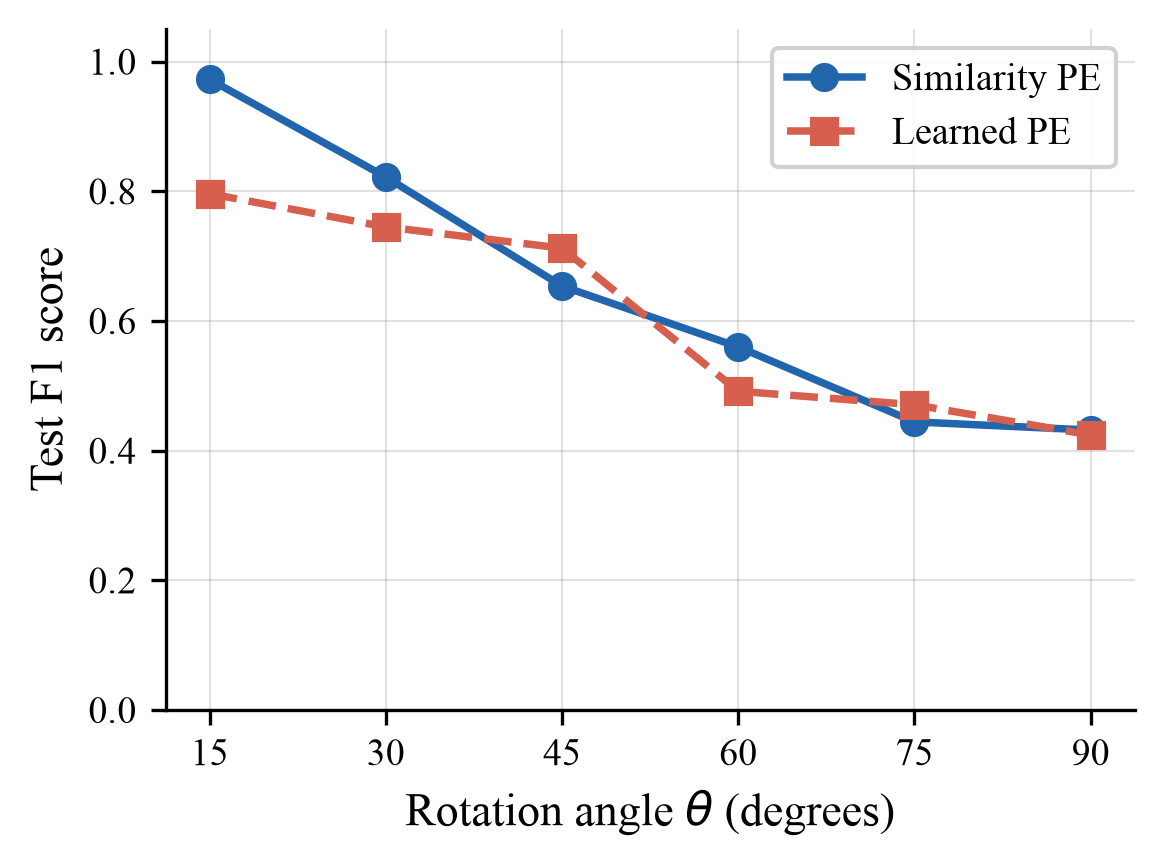}
        \caption{Test F1 score}
    \end{subfigure}
    \vspace{0.5em}
    \begin{subfigure}[b]{0.48\textwidth}
        \includegraphics[width=\textwidth]{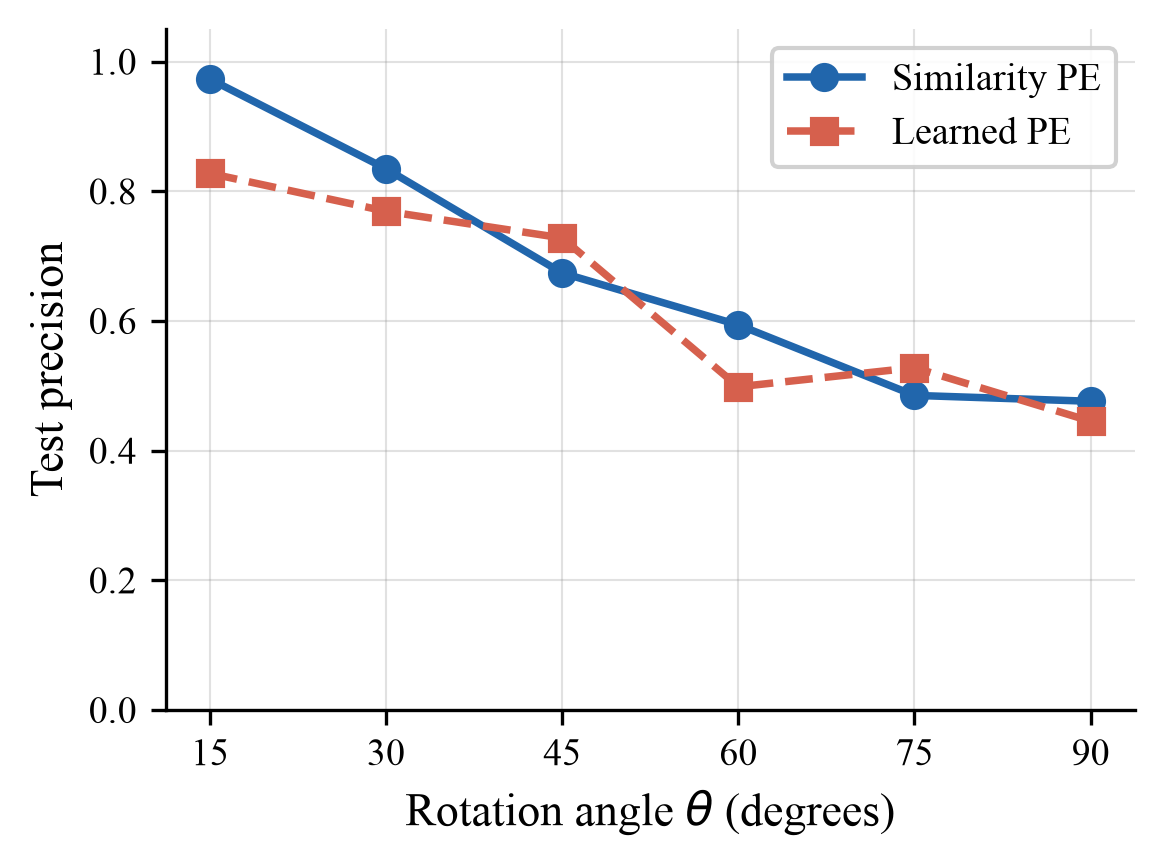}
        \caption{Test precision}
    \end{subfigure}
    \hfill
    \begin{subfigure}[b]{0.48\textwidth}
        \includegraphics[width=\textwidth]{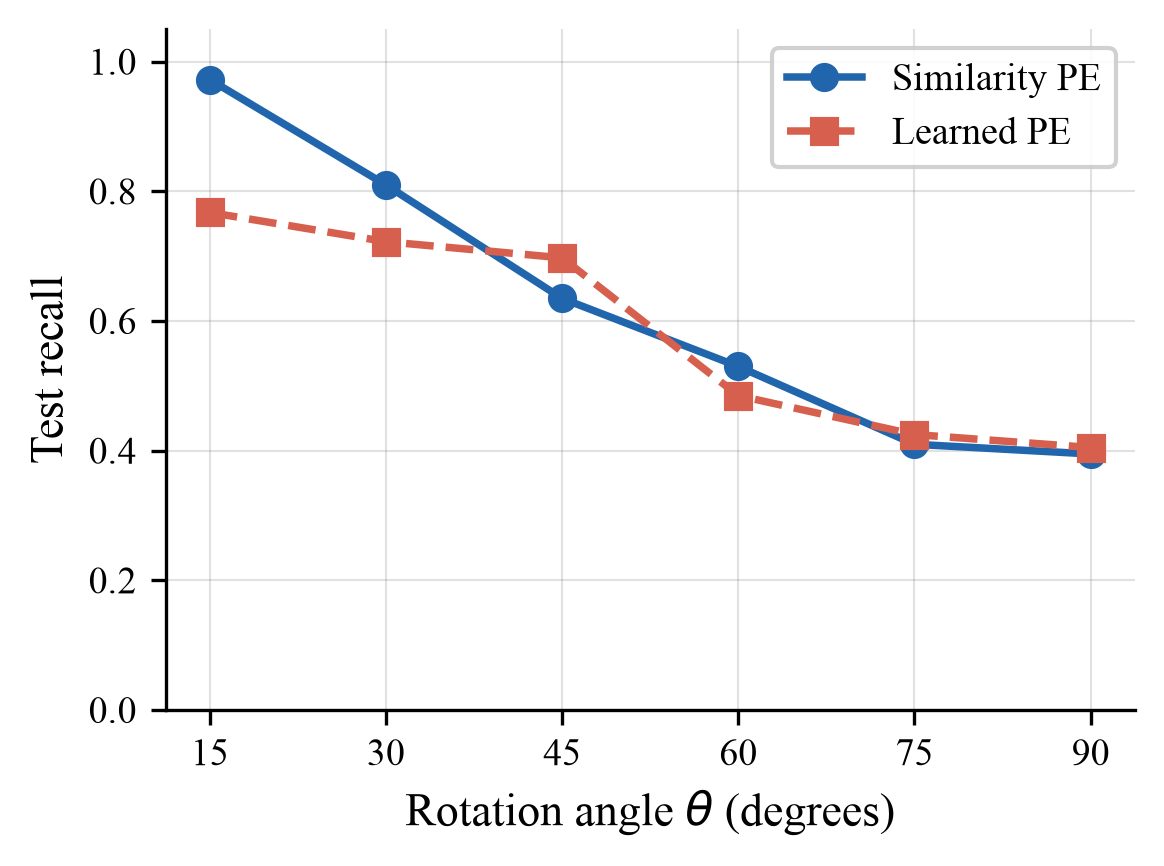}
        \caption{Test recall}
    \end{subfigure}
    \caption{Performance on the Shapes dataset as a function of test rotation angle (degrees). simPE strongly outperforms Learned PE at small angles; the advantage narrows at larger angles and the two methods converge near $45^\circ$--$90^\circ$.}
    \label{fig:shapes-results}
\end{figure}

\begin{table}[H]
\centering
\caption{Shapes dataset: test accuracy (\%) and F1 score at each rotation angle.}
\label{tab:shapes}
\setlength{\tabcolsep}{6pt}
\begin{tabular}{lccccccc}
\toprule
\multirow{2}{*}{Method} & \multicolumn{6}{c}{Rotation angle (degrees)} \\
\cmidrule(lr){2-7}
 & $15^\circ$ & $30^\circ$ & $45^\circ$ & $60^\circ$ & $75^\circ$ & $90^\circ$ \\
\midrule
\multicolumn{7}{l}{\textit{Accuracy (\%)}} \\
simPE       & \textbf{97.25} & \textbf{81.00} & 63.50          & \textbf{53.00} & 41.00          & 39.50 \\
Learned PE  & 76.75          & 72.25          & \textbf{69.75} & 48.50          & \textbf{42.50} & \textbf{40.50} \\
\midrule
\multicolumn{7}{l}{\textit{F1 score}} \\
simPE       & \textbf{0.973} & \textbf{0.822} & 0.654          & \textbf{0.560} & 0.445          & 0.432 \\
Learned PE  & 0.797          & 0.745          & \textbf{0.713} & 0.492          & \textbf{0.471} & \textbf{0.424} \\
\bottomrule
\end{tabular}
\end{table}

The Shapes dataset shows a clear two-regime behaviour (Figure~\ref{fig:shapes-results}, Table~\ref{tab:shapes}). At $15^\circ$, simPE achieves 97.25\% accuracy compared to 76.75\% for Learned PE, a gap of 20.5 percentage points---the largest first-angle advantage observed across all four datasets. At $30^\circ$, simPE (81.00\%) remains substantially ahead of Learned PE (72.25\%). However, starting at $45^\circ$ the advantage disappears: Learned PE (69.75\%) overtakes simPE (63.50\%), a crossover analogous to that observed for FashionMNIST at $30^\circ$. Beyond $45^\circ$, both methods continue to degrade at comparable rates, converging near chance level at $90^\circ$ (39.5\% vs.\ 40.5\% for a 4-class problem). The F1, precision, and recall curves confirm this pattern: all metrics show a strong simPE advantage at small angles that erodes at large angles. The crossover is consistent with the theoretical bound of Corollary~\ref{cor:small-angle}: as $|\theta|$ grows, the perturbation $\|R_\theta X - X\|_F$ increases, and the stability guarantee, though valid for all angles, no longer suffices to offset the large mismatch between the training and test distributions.

\subsection{Summary}

Table~\ref{tab:summary} provides a high-level summary of the accuracy advantage of simPE over Learned PE across all datasets and the first (smallest) rotation angle evaluated. In all three datasets, simPE starts with a higher accuracy and maintains the advantage for the majority of the rotation range tested.

\begin{table}[H]
\centering
\caption{Summary of simPE vs.\ Learned PE accuracy (\%) at the smallest tested rotation angle across all four datasets.}
\label{tab:summary}
\begin{tabular}{lcccc}
\toprule
Dataset & Rotation & simPE & Learned PE & Difference \\
\midrule
Arrow        & $5^\circ$  & 99.50 & 98.00 & $+$1.5 \\
Digits       & $15^\circ$ & 86.20 & 66.30 & $+$19.9 \\
FashionMNIST & $5^\circ$  & 84.69 & 76.69 & $+$8.0 \\
Shapes       & $15^\circ$ & 97.25 & 76.75 & $+$20.5 \\
\bottomrule
\end{tabular}
\end{table}

\section{Discussion}

The analysis and experiments developed above highlight an important distinction between invariance and robustness.

\paragraph{Invariance vs.\ robustness.}
Exact invariance is a strong property requiring that the positional encoding remain unchanged under rotation. simPE generally does not satisfy this requirement (Proposition~\ref{prop:not-invariant}). However, in many practical scenarios, exact invariance is neither necessary nor desirable: orientation often carries useful discriminative information, and discarding it entirely would reduce the encoding's expressiveness \cite{cohen2016group}. What is needed instead is a weaker form of stability ensuring that small geometric perturbations induce only small changes in the encoding. This is precisely what our theoretical results establish (Theorem~\ref{thm:simpe-lipschitz}, Corollary~\ref{cor:small-angle}).

\paragraph{Theoretical mechanism.}
The key mechanism underlying robustness is the Lipschitz continuity of the constituent components of simPE. Under mild regularity assumptions---bounded feature domain, Lipschitz similarity, Lipschitz post-processing, and non-degenerate normalization---the composition $\Phi \circ S$ is itself Lipschitz, and its response to rotations is bounded by the product of the individual Lipschitz constants and the magnitude of the rotation. This is a general principle: any encoding built from Lipschitz components inherits robustness to small perturbations in the input.

\paragraph{Experimental corroboration.}
The experimental results strongly support the theoretical predictions. Across all four datasets, simPE exhibits a slower and more graceful degradation curve than Learned PE as rotation angle increases, particularly in the small-to-moderate angle regime. On the Arrow synthetic dataset, the gap reaches almost 50 percentage points at $25^\circ$. On the Digits benchmark, simPE nearly doubles Learned PE accuracy at $90^\circ$. On the Shapes dataset, simPE leads by 20.5 percentage points at $15^\circ$, with the advantage narrowing at larger angles and reversing near $45^\circ$. On FashionMNIST, simPE is superior for rotations up to $15^\circ$ (the regime of practical interest), with Learned PE recovering slightly only at the largest tested angle ($30^\circ$). In both FashionMNIST and Shapes, the crossover at large angles is consistent with dataset-specific structure (orientation-symmetric clothing categories and the degradation of training-distribution proximity at large angular offsets, respectively) rather than a failure of the theoretical stability guarantee.

\paragraph{Implications for medical imaging.}
This distinction is especially important in medical imaging, the original domain for which simPE was conceived \cite{leonardi2024simpe, 10.1007/978-3-032-16708-8_6}. In this setting, small rotations are not rare or artificial perturbations; they arise naturally from the acquisition process itself. Instrument configuration, patient placement, and subtle differences in imaging setup can all introduce small rotational discrepancies across scans. The relevant theoretical question is not whether simPE is exactly invariant, but whether it remains stable under the kinds of perturbations that occur in practice. Our results confirm that this is indeed the case.

\paragraph{Improving stability.}
Our results also suggest that the stability of simPE can be improved by explicitly controlling the Lipschitz constants of its components: bounded similarity kernels, regularized projections, or normalization schemes that avoid near-zero singularities will all reduce the stability constant $C = 2L_\Phi L_\kappa(M)\sqrt{N}$ and thus improve robustness. This provides a principled roadmap for architectural refinements.

\paragraph{Limitations and future directions.}
While the present analysis is intentionally general, it can be extended in several directions. First, the analysis covers arbitrary rotations but the experimental validation focuses on 2D image rotation; extending to 3D volumetric medical data is a natural next step. Second, the Lipschitz constants $L_\Phi$ and $L_\kappa(M)$ are not explicitly estimated here; tighter bounds for specific similarity choices (e.g., dot-product, cosine, RBF kernels) would make the stability guarantee more actionable. Third, investigating how the bounds interact with learned feature extractors upstream of simPE---especially when those extractors are themselves only approximately equivariant---remains an open problem. Finally, extending the controlled evaluation to other geometric transformations (translations, anisotropic scalings, affine perturbations) would provide a more comprehensive picture of simPE's geometric robustness.

\section{Conclusion}

This paper presented a combined theoretical and experimental study of the robustness of similarity-based positional encoding under rotations. On the theoretical side, we showed that simPE is not rotation-invariant in general, but remains robustly stable to rotational perturbations under mild Lipschitz regularity assumptions on its components. We derived explicit perturbation bounds and a global estimate in Frobenius norm that separates the architectural Lipschitz constant from the geometric perturbation magnitude. On the experimental side, we validated these findings on four controlled datasets---Arrow (synthetic), Shapes (synthetic), Digits (synthetic), and FashionMNIST (benchmark)---in which training is always performed on canonical-orientation images and evaluation probes increasing rotation angles. Across all datasets and metrics, simPE consistently and substantially outperforms learned positional encoding under rotation, particularly in the small-to-moderate angle regime that is most relevant to practical applications such as medical imaging.

These results provide a concise mathematical characterization of simPE as a stable, though not invariant, positional encoding mechanism. They are particularly relevant for medical imaging, where simPE was originally introduced and where small acquisition-induced rotations naturally occur due to instruments and imaging procedures. Future work will include extending the analysis to three-dimensional settings, investigating tighter stability constants for specific similarity choices, and validating the theoretical predictions on medical imaging benchmarks.

\bibliographystyle{plain}
\bibliography{references}

\end{document}